\pdfoutput=1
\documentclass{article}

\usepackage[preprint]{neurips_2026}
\setcitestyle{round,semicolon}

\usepackage[T1]{fontenc}
\usepackage{microtype}
\usepackage{amsmath}
\usepackage{amssymb}
\usepackage{amsthm}
\usepackage{graphicx}
\usepackage{url}
\usepackage{booktabs}
\usepackage{xcolor}
\usepackage[algo2e,ruled]{algorithm2e}
\usepackage{hyperref}
\hypersetup{colorlinks,
            linkcolor=blue,
            citecolor=blue,
            urlcolor=magenta,
            pdftitle={On the Geometry of Multimodal Saturation: Riemannian VICReg},
            pdfauthor={Nessim Ben Abbes, Duc Han Le, Sabri Mtibaa, Van-Tam Nguyen}}

\newtheoremstyle{headplain}{}{}{\itshape}{}{\bfseries}{}{\newline}{\thmname{#1}\thmnumber{ #2}\thmnote{ (#3)}}
\newtheoremstyle{headdefinition}{}{}{\normalfont}{}{\bfseries}{}{\newline}{\thmname{#1}\thmnumber{ #2}\thmnote{ (#3)}}
\theoremstyle{headplain}
\newtheorem{theorem}{Theorem}
\newtheorem{lemma}[theorem]{Lemma}
\newtheorem{proposition}[theorem]{Proposition}
\newtheorem{corollary}[theorem]{Corollary}
\theoremstyle{headdefinition}
\newtheorem{definition}[theorem]{Definition}

\newenvironment{keywords}
  {\begin{quote}\small\textbf{Keywords:}\ \ignorespaces}
  {\end{quote}}


\newcommand{\rvicreg}{R-VICReg}
\newcommand{\vicreg}{VICReg}
\newcommand{\simsiam}{SimSiam}
\newcommand{\Tri}{(I,T,\mathrm{Tab})}
\newcommand{\Bstar}{B^{*}}
\newcommand{\TriGain}{\Delta_{\mathrm{tri}}}
\newcommand{\TriGainOf}[1]{\Delta_{\mathrm{tri}}^{\text{#1}}}
\newcommand{\Npack}{\mathcal{N}_{\mathrm{pack}}}

\DeclareMathOperator{\FM}{FM}
\DeclareMathOperator{\diam}{diam}
\newcommand{\norm}[1]{\left\lVert #1\right\rVert}
\newcommand{\ip}[2]{\langle #1,#2\rangle}

\title{On the Geometry of Multimodal Saturation:\\Riemannian VICReg}

\author{Nessim Ben Abbes\\
  Protectline\\
  \texttt{nessim.benabbes@protectline.fr}
  \And
  Duc Han Le\\
  T\'el\'ecom Paris\\
  \texttt{duc-han.le@telecom-paris.fr}
  \And
  Sabri Mtibaa\\
  Protectline\\
  \texttt{sabri.mtibaa@protectline.fr}
  \And
  Van-Tam Nguyen\\
  T\'el\'ecom Paris\\
  \texttt{van-tam.nguyen@telecom-paris.fr}}

\begin{document}

\maketitle

\begin{abstract}
In self-supervised learning, a third modality should improve, or at least
preserve, performance. Across nine image--text--tabular datasets, we show that
it instead harms performance: the trimodal model underperforms its own best
bimodal subset in $55.6\%$ of paired runs under \vicreg. The same
failure occurs in $51.1\%$ of paired runs under \simsiam. We call this failure
multimodal saturation. We propose that the failure lies in the alignment
geometry. Riemannian \vicreg\ (\rvicreg) generalizes classical \vicreg: it
aligns views by squared geodesic distance on learnable negative-curvature
product factors and recovers \vicreg\ exactly as curvature vanishes. Over the
same $45$ paired runs, \rvicreg\ raises the probability that the third modality
helps from $44.4\%$ to $64.4\%$, with gains concentrated where
\vicreg\ saturates.
\end{abstract}

\begin{keywords}
multimodal self-supervised learning, hyperbolic geometry, product manifolds,
readout capacity, packing bounds, VICReg, representational geometry
\end{keywords}

\section{Introduction}
\label{sec:intro}

Multimodal self-supervised learning (SSL) inherits a monotonicity expectation
from the scaling literature: an additional modality should improve, or at worst
preserve, representation
quality~\citep{baltrusaitis2019survey,girdhar2023imagebind}.
The expectation is reasonable: a trimodal encoder observes strictly more than
any bimodal subset. Yet in masked-view non-contrastive multimodal SSL, after
identical pretraining and finetuning, the trimodal model underperforms its own
best bimodal subset in $55.6\%$ of paired runs under
\vicreg~\citep{bardes2022vicreg} and $51.1\%$ under
\simsiam~\citep{chen2021simsiam}. Figure~\ref{fig:datasets}
reports the per-dataset rates behind these
aggregates. We call this failure multimodal saturation. Related
non-monotonic effects are usually attributed to optimization: gradient
imbalance, dominant modalities, fusion
schedules~\citep{wu2022greedy,magal2025negative,chaudhuri2025mmcollapse,
xu2025balancebenchmark}.

The optimization account is useful but incomplete. It does not identify which
component of the objective carries the failure. We propose that the missing
variable is the metric inside the invariance term.

In non-contrastive SSL, the invariance term contracts embeddings of two views of
the same input toward each
other~\citep{grill2020byol,zbontar2021barlow,wang2020alignment}. Under heavy
masking or modality dropout, partial observations can lose the evidence that
distinguishes two semantic
states~\citep{he2022mae,ebrahimi2023lanistr,ma2022missingmodality}. The metric
for this contraction also measures downstream separability. The invariance
metric sets alignment tightness and readout capacity. Simply put, it
determines how many semantic states a downstream head can distinguish at its
margin.

In Euclidean space, readout capacity grows only polynomially with radius. When
demand exceeds that capacity, distinct states collide at the margin scale
(Figure~\ref{fig:prime}). The added modality becomes redundant at the
representation level while remaining informative at the input level. We name
this mechanism induced redundancy: the alignment geometry manufactures
redundancy that the data did not contain.

\begin{figure}[tb]
\centering
\includegraphics[width=\linewidth]{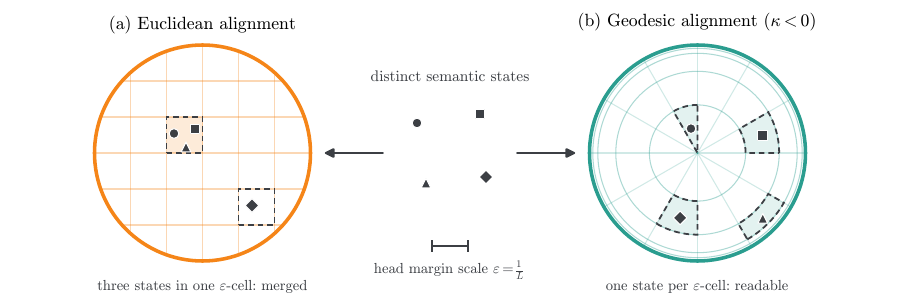}
\caption{The alignment geometry decides whether aligned states stay readable.
Masked-view alignment contracts each semantic state (one marker shape per state)
inside the same radius budget, and a downstream head resolves states only at its
margin scale $\varepsilon=\frac{1}{L}$. (a)~Under Euclidean alignment, several
states land in one $\varepsilon$-cell and merge: induced redundancy.
(b)~Under geodesic alignment at $\kappa<0$, the same budget offers exponentially
more $\varepsilon$-cells, so each state keeps its own.
Lemmas~\ref{lem:euclid-packing}--\ref{lem:hyperbolic-upper} prove the two growth
laws.}
\label{fig:prime}
\end{figure}

The diagnosis suggests a targeted intervention rather than a new objective.
Constant negative curvature buys exponential volume growth in radius and
therefore exponential readout capacity (Corollary~\ref{cor:bits}; proofs in
Appendix~\ref{app:prelim}). We hold the pipeline fixed (encoders, fusion,
masking, projector dimension, regularization weights, finetuning) and replace
only the metric of the invariance term.

Riemannian \vicreg\ (\rvicreg) generalizes \vicreg. It maps the projector
output onto a product of low-dimensional constant negative-curvature factors
with learnable curvature, aligns views by squared product-geodesic distance, and
computes VICReg-style variance and covariance penalties in tangent coordinates
at batch Fr\'echet means. The full objective has an exact VICReg limit
(Proposition~\ref{prop:euclid-limit}). Curvature is one continuous, learnable
geometric degree of freedom, not an architectural commitment.

While \rvicreg\ is not a drop-in replacement for \vicreg, we find that its gains
concentrate where the Euclidean objective is capacity-limited. That selectivity
is the falsifiable footprint we test.

\paragraph{Related work.}
\label{sec:related}
Joint-embedding SSL without negatives first avoided collapse through
architectural asymmetry (BYOL, \simsiam)~\citep{grill2020byol,chen2021simsiam},
then through variance--covariance constraints (Barlow Twins,
\vicreg)~\citep{zbontar2021barlow,bardes2022vicreg}, and analyses of
alignment--uniformity, spectra, and information content
followed~\citep{wang2020alignment,mialon2022vcreg,balestriero2022contrastive,
shwartzZiv2023information}.
Masked modeling extended the recipe from language and
vision~\citep{devlin2018bert,he2022mae,xie2022simmim} to multimodal
encoders~\citep{geng2022m3ae,bachmann2022multimae} and to the LANISTR-style
image--text--tabular scaffold we adopt~\citep{ebrahimi2023lanistr}; as
modalities accumulated, non-monotonic gains appeared and were attributed to
optimization: missing-modality sensitivity, collapse, greedy reliance, negative
co-learning~\citep{ma2022missingmodality,chaudhuri2025mmcollapse,wu2022greedy,
magal2025negative}.
In parallel, hyperbolic representation learning grew from hierarchy
embedding~\citep{nickel2017poincare,ganea2018hyperbolicnn} to product and
mixed-curvature spaces with learnable per-factor
curvature~\citep{gu2019mixedcurvature,skopek2020mixedcurvature,
bachmann2020ccgcn},
entered contrastive and cross-modal
SSL~\citep{ge2023hcl,desai2023hyperbolic,franco2023hysp,durrant2023hmsn}, and
reached geometry-adaptive \vicreg\ variants that kernelize the objective or
penalize curvature~\citep{sepanj2025kernelvicreg,ghojogh2025curvssl}. The two
lines have not met at the invariance term: multimodal responses change balancing
schedules, fusion, or regularizers and keep the alignment metric Euclidean,
while the hyperbolic line targets hierarchy rather than readout capacity at the
downstream margin. That capacity vocabulary comes from metric entropy and
large-margin classification in metric
spaces~\citep{kolmogorov1961epsilon,vonluxburg2004distance,
gottlieb2014efficient}.

\paragraph{Contributions.}
We give (i) multimodal saturation, documented over nine datasets and two
Euclidean objectives; (ii) a proposed account, induced redundancy,
formalized by a readout-capacity theorem (Theorem~\ref{thm:capacity})
placing readable concepts between $\Npack(B,\frac{1}{L}+2\delta)$ and
$\Npack(B,\frac{1}{L})$, so the excess above capacity counts forced merges;
(iii) \rvicreg, a generalization of \vicreg\ with geodesic invariance on
a learnable-curvature product manifold, Fr\'echet-centered tangent
regularization, and an exact \vicreg\ limit; and (iv) experimental results
that report what the metric change does to multimodal performance.

\begin{figure}[tb]
\centering
\includegraphics[width=\linewidth]{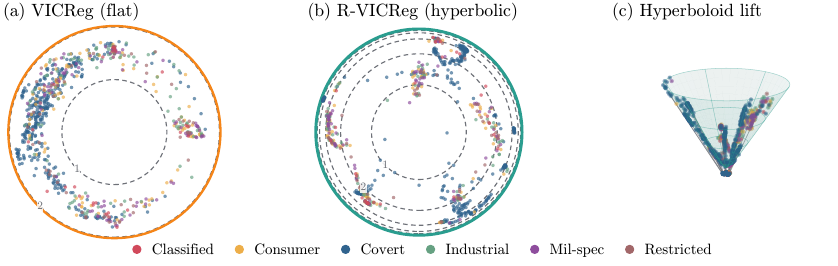}
\caption{Three views render the learned metric of one trimodal CSGO run (seed
456). (a) In the flat \vicreg\ projector, equal-distance rings stay evenly
spaced. (b) In the most-utilized \rvicreg\ factor ($\kappa_{32}=-0.988$), the
rings at the same geodesic distances crowd toward the boundary. (c) The same
points sit on the hyperboloid above their Poincar\'e-ball shadow (construction
in Appendix~\ref{app:geometry}).}
\label{fig:geometry-showcase}
\end{figure}

\section{Saturation as a Readout-Capacity Deficit}
\label{sec:theory}

We read multimodal saturation as a margin-scale failure: after finetuning,
distinct semantic states are no longer separated enough for a downstream head to
discriminate them robustly. The invariance metric pulls views of one state
together and measures how far different states remain; both roles enter the
capacity count.

Throughout, $(\mathcal M,d)$ is the representation space. We suppose codes lie
in a closed ball $B(o,R)$; normalization and weight decay motivate this radius
budget. $\Npack$ denotes packing numbers (Definition~\ref{def:packing}). A
semantic state is realized as a set of codes, and the invariance term shrinks
that set's diameter.

\begin{definition}[Aligned concepts and readability]
\label{def:family}
Let $K\in\mathbb N$ and $\delta\ge0$. A $\delta$-aligned $K$-concept family in
$B(o,R)$ is a collection of nonempty sets $C_1,\dots,C_K\subseteq B(o,R)$ with
$\diam(C_i)\le\delta$ for every $i\in\{1,\dots,K\}$.\\
The family is $(L,1)$-readable if there exist $L$-Lipschitz functions
$h_1,\dots,h_K:\mathcal M\to\mathbb R$ such that, for all $i\in\{1,\dots,K\}$
and all $c\in C_i$,
\[
h_i(c)\;\ge\;\max_{j\in\{1,\dots,K\}\setminus\{i\}}h_j(c)+1.
\]
\end{definition}

\begin{theorem}[Readout capacity]
\label{thm:capacity}
Fix $o\in\mathcal M$, $R>0$, $L>0$, and $\delta\ge0$.
\begin{enumerate}
\item[(i)] Obstruction: if a $\delta$-aligned $K$-concept family in $B(o,R)$
is $(L,1)$-readable, then
\[
K\;\le\;\Npack\big(B(o,R),\tfrac{1}{L};d\big).
\]
\item[(ii)] Guarantee: if $s_1,\dots,s_K\in B(o,R)$ satisfy
$d(s_i,s_j)\ge\tfrac{1}{L}+2\delta$ for all $i\neq j$, then every
$\delta$-aligned $K$-concept family with $s_i\in C_i$ for every
$i\in\{1,\dots,K\}$ is $(L,1)$-readable.
\end{enumerate}
\end{theorem}

The proof is in Appendix~\ref{app:prelim}.

Exceeding the packing capacity forces merges. If $K$ concepts are realized as a
$\delta$-aligned family and $K>\Npack(B(o,R),\frac{1}{L})$, then under any
$L$-Lipschitz heads at most $\Npack(B(o,R),\frac{1}{L})$ of them attain unit
margin. This follows by applying (i) to the subfamily on which the margin holds.
At least $K-\Npack(B(o,R),\frac{1}{L})$ concepts therefore fail on some view.
Induced redundancy is this count: redundancy manufactured by the
geometry, not present in the data. Saturation, in this lens, is not noise. It is
pigeonholing.

\paragraph{Curvature sets the growth law.}
The theorem sandwiches the readable count between two packings of one ball, at
scales that differ by $2\delta$. Geometry enters through one quantity:
ball-packing growth in radius. Appendix~\ref{app:prelim} proves matching bounds
for both geometries:
\begin{equation}
\Npack\big(B_{\mathbb
R^n}(0,R),\varepsilon\big)\le\Big(1+\tfrac{2R}{\varepsilon}\Big)^{n},
\qquad
\Npack\big(B_{\mathbb H^n_\kappa}(R),\varepsilon\big)\ge
c_{\varepsilon}\,e^{(n-1)\sqrt{-\kappa}\,R}.
\end{equation}

The Euclidean bound grows polynomially in $R$; the hyperbolic bound grows
exponentially. Both bounds are two-sided
(Lemmas~\ref{lem:euclid-packing}--\ref{lem:hyperbolic-upper}). Products inherit
the exponential lower bound factor-wise, so bits add across factors
(Proposition~\ref{prop:product-packing}).

\begin{corollary}[Radius cost of bits]
\label{cor:bits}
Fix $L>0$ and $\delta\ge0$.
\begin{enumerate}
\item[(i)] In $\mathbb R^n$ the readable-concept count grows polynomially in
$R$: hosting $\log K$ bits needs a radius exponential in $(\log K)/n$.
\item[(ii)] In $\mathbb H^n_\kappa$ with $\kappa<0$ it grows exponentially in
$R$: the required radius is linear in $\log K$.
\item[(iii)] At fixed $(n,R,L,\delta)$, replacing the flat metric by curvature
$\kappa<0$ multiplies the guaranteed count by a factor that grows
exponentially in $R$.
\end{enumerate}
\end{corollary}

The proof is in Appendix~\ref{app:prelim}.

Demand compounds while flat supply crawls. Compositional states multiply across
modalities, so demanded bits add per modality. A third modality is where
additive demand can cross a logarithmic flat supply. Tightening alignment
narrows the sandwich, but only the metric changes the growth law. The invariance
term of \rvicreg\ shrinks $\delta$ in the metric whose capacity grows
exponentially. Its variance and covariance terms prevent degenerate collapse
into a vanishing region that would destroy capacity in any geometry.

The theorem gives both an obstruction and a guarantee. It does not certify that
a given dataset's demand exceeds the Euclidean supply. The saturation diagnostic
addresses that question.

The supply side leaves a visible signature: curvature buys capacity, alignment
spends it, and the codes crowd into the exponential-volume shell.
Figure~\ref{fig:geometry-showcase} shows the signature on a trained run: $71\%$
of codes sit beyond geodesic radius $2$.

The design principle has a falsifiable footprint: help where the flat objective
saturates and impose mild cost where it does not. Section~\ref{sec:results}
tests this footprint.

\section{Riemannian VICReg}
\label{sec:method}

\paragraph{Masked-view scaffold.}
We adopt a LANISTR-style pretrain-then-finetune pipeline for image ($I$), text
($T$), and tabular ($\mathrm{Tab}$) inputs~\citep{ebrahimi2023lanistr}. Each
minibatch yields an unmasked view $a_1(x)$ and a masked view $a_2(x)$ via
modality-specific masking; modality encoders feed single-layer projections into
a fusion Transformer, and a projector $g$ maps the fused CLS representations to
codes $z_1,z_2\in\mathbb R^D$. Pretraining combines weighted unimodal
masked-reconstruction losses with a masked-view alignment loss
$\mathcal L_{\mathrm{MMM}}$; the compared objectives (\simsiam, \vicreg,
\rvicreg) instantiate only $\mathcal L_{\mathrm{MMM}}$, and finetuning attaches
a classifier to the fusion CLS embedding. Hyperparameters and the minibatch
algorithm are in Appendix~\ref{app:hyperparams}.

\subsection{Geodesic invariance on a product of curved factors}

\rvicreg\ represents the code $z_i\in\mathbb R^D$ as concatenated tangent
vectors. We partition it into $F$ factors $z_i^{(f)}\in\mathbb R^{d}$ with
$D=Fd$. Each factor carries a learnable curvature $\kappa_f<0$, and the
exponential map at the origin sends it to the $\kappa$-stereographic model:
$x_i^{(f)}=\exp_0^{\kappa_f}(z_i^{(f)})$. Closed-form $\exp$, $\log$, M\"obius
addition, and geodesic distance are collected in Table~\ref{tab:kappa-ops}
(Appendix~\ref{app:kappa-ops}).

The curved factors are not an arbitrary choice. We read each factor's code as
a location paired with its alignment uncertainty; the metric invariant under
sufficient statistics is then the Fisher--Rao metric, and for isotropic
Gaussian factors it is hyperbolic (Appendix~\ref{app:fisher}). Euclidean
\vicreg\ is the slice on which the uncertainty coordinate is frozen, and
\rvicreg\ unfreezes that coordinate with one learnable curvature per factor.

The product manifold is constant-curvature factor-wise, not globally; distances
combine by the product metric, and Proposition~\ref{prop:product-packing} bounds
its packing capacity from below.

The invariance term replaces Euclidean mean squared error with normalized
squared geodesic distance,
\begin{equation}
\label{eq:rvicreg-inv}
\mathcal L_{\mathrm{inv}}
:= \mathbb
E\Bigg[\frac{1}{F}\sum_{f=1}^{F}\frac{d_{\kappa_f}\big(x_1^{(f)},
x_2^{(f)}\big)^{2}}{4d}\Bigg],
\end{equation}
where the $4$ matches the stereographic convention $d_0(x,y)^2=4\norm{x-y}_2^2$
and $1/(Fd)=1/D$ matches \vicreg's normalization.

\begin{figure}[tb]
\centering
\includegraphics[width=\linewidth]{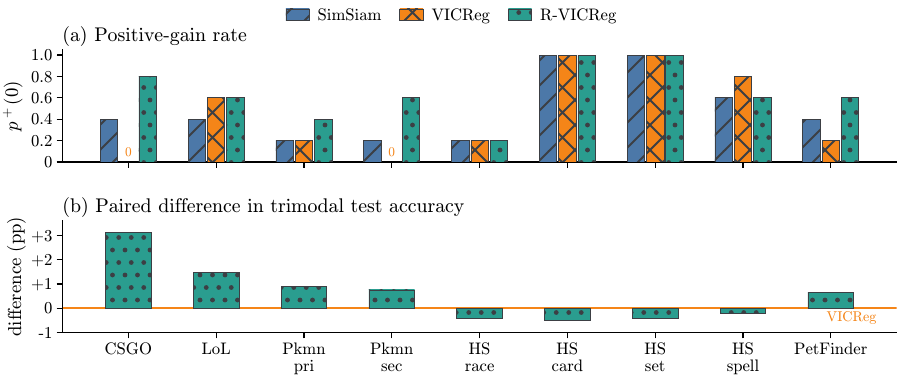}
\caption{Panel (a) reports the per-dataset positive-gain rate
$p^{+}(0)=\Pr(\TriGain>0)$ over five seeds, so rates are multiples of $0.2$.
\vicreg\ scores $0$ of $5$ on CSGO and Pokemon Secondary (marked $0$). Panel
(b) reports the paired difference in trimodal test accuracy between
\rvicreg\ and \vicreg: bars give the five-seed mean and the orange line the
\vicreg\ baseline.}
\label{fig:datasets}
\end{figure}

\subsection{Variance and covariance in Fr\'echet-centered tangent coordinates}

Variance and covariance statistics require a vector space, but a curved factor
is not one. For each factor, with $q_1,\dots,q_{2B}$ the factor codes of both
views in a batch of size $B$, the batch Fr\'echet mean $\mu_f :=
\FM\big(\{q_j\}_{j=1}^{2B}\big)$ is the minimizer of
$\sum_{j}d_{\kappa_f}(\mu,q_j)^{2}$. We approximate it by three
origin-initialized Karcher iterations~\citep{karcher1977center},
\[
\mu\;\leftarrow\;\exp_\mu^{\kappa_f}\Big(\frac{1}{2B}\sum_{j=1}^{2B}\log_\mu^{\kappa_f}(q_j)\Big),
\]
and detach the result. We then log-map both views to the tangent space at
$\mu_f$ and concatenate across factors,
\[
t_{i,b}^{(f)}\;:=\;\log_{\mu_f}^{\kappa_f}\big(x_{i,b}^{(f)}\big),
\qquad
t_{i,b}\;:=\;\big[t_{i,b}^{(1)};\dots;t_{i,b}^{(F)}\big]\in\mathbb R^{D}.
\]
The standard \vicreg\ penalties act on these tangent codes,
\begin{equation}
\label{eq:rvicreg-varcov}
\mathcal L_{\mathrm{var}} := \frac{1}{2}\sum_{i\in\{1,2\}}\mathbb
E\big[\mathrm{ReLU}\big(\gamma-\mathrm{std}(t_i)\big)\big],
\qquad
\mathcal L_{\mathrm{cov}} := \sum_{i\in\{1,2\}}\frac{1}{D}\sum_{a\neq
b}\mathrm{Cov}(t_i)_{ab}^{2}.
\end{equation}
Here $\mathrm{std}(t_i)$ is the per-coordinate standard deviation over the
batch, the expectation in $\mathcal L_{\mathrm{var}}$ averages the $D$
coordinates, and $\mathrm{Cov}(t_i)$ is the $D\times D$ batch covariance of the
tangent codes. The variance term prevents individual tangent coordinates from
collapsing; the covariance term limits redundancy between coordinates. The full
objective is $\mathcal L_{\mathrm{rvicreg}}=\lambda_{\mathrm{inv}}\mathcal
L_{\mathrm{inv}}+\lambda_{\mathrm{var}}\mathcal
L_{\mathrm{var}}+\lambda_{\mathrm{cov}}\mathcal L_{\mathrm{cov}}$ with the exact
\vicreg\ weighting
$(\lambda_{\mathrm{inv}},\lambda_{\mathrm{var}},\lambda_{\mathrm{cov}})=(25,25,
1)$,
$\gamma=1$.

\begin{proposition}[Exact VICReg limit]
\label{prop:euclid-limit}
Fix a minibatch of codes and let $\kappa_f\to0$ for every $f\in\{1,\dots,F\}$.
\begin{enumerate}
\item[(i)] $\exp_0^{\kappa_f}(z_i^{(f)})\to z_i^{(f)}$ and
$d_{\kappa_f}(x_1^{(f)},x_2^{(f)})^{2}\to4\norm{z_1^{(f)}-z_2^{(f)}}_2^{2}$,
so $\mathcal L_{\mathrm{inv}}\to\frac{1}{D}\norm{z_1-z_2}_2^{2}$ in the batch
average.
\item[(ii)] Every Karcher iterate $\mu_f$ converges to the Euclidean batch
mean and $\log_{\mu_f}^{\kappa_f}$ to subtraction, so the tangent codes
converge to the centered codes.
\item[(iii)] Variance and covariance statistics are translation-invariant, so
$\mathcal L_{\mathrm{var}}$ and $\mathcal L_{\mathrm{cov}}$ converge to their
\vicreg\ values.
\end{enumerate}
Hence $\mathcal L_{\mathrm{rvicreg}}\to\mathcal L_{\mathrm{vicreg}}$
(Lemma~\ref{lem:flat-limit}, Appendix~\ref{app:kappa-ops}).
\end{proposition}

\vicreg\ is the boundary case of the whole objective: the family is continuous
in the geometry, and \vicreg's loss weights transfer unchanged.

\paragraph{Learnable curvature and stability.}
Per-factor curvature is parameterized as
$\kappa_f=-\kappa_{\max}\,\sigma(\theta_f)$ with $\kappa_{\max}=1.5$, so
$\kappa_f\in(-\kappa_{\max},0)$ while $\theta_f\in\mathbb R$ trains by
backpropagation; we initialize $\theta_f$ by inverse sigmoid to target
$\kappa_0=-1$ with small per-factor jitter. For stability in the Poincar\'e ball
of radius $1/\sqrt{-\kappa_f}$, we clip tangent norms before the exponential map
and project points exceeding the boundary. We monitor boundary
utilization $u=\sqrt{-\kappa_f}\,\norm{x^{(f)}}_2\in[0,1)$: values near $1$
mean embeddings occupy the high-volume region where the exponential capacity is
spent.

\section{Experimental Design}
\label{sec:experiments}

\paragraph{Datasets and modality sweep.}
We use nine trimodal image--text--tabular classification datasets: eight from
MuG-Bench~\citep{lu2023mug}, namely CSGO ($n{=}767$), LoL ($n{=}1{,}006$),
Pokemon Primary and Secondary type ($n{=}719$ each), Hearthstone Race
($8{,}568$), CardClass ($13{,}153$), Set ($13{,}087$), and Spell ($3{,}924$),
plus PetFinder ($n{=}11{,}996$ training examples, $80/5/15$
split)~\citep{kaggle2019petfinder}.
For each dataset and method, we train all seven modality subsets
$\{I,T,\mathrm{Tab},(I,T),(I,\mathrm{Tab}),(T,\mathrm{Tab}),\Tri\}$ under the
same pipeline. These subsets measure what the third modality adds over every
bimodal alternative.

\paragraph{Saturation metrics.}
For each run (dataset $\times$ seed), let $S(\cdot)$ be downstream test
accuracy. The best bimodal baseline is $\Bstar :=
\max\{S(I,T),\,S(I,\mathrm{Tab}),\,S(T,\mathrm{Tab})\}$, and the trimodal
marginal gain is $\TriGain := S(\Tri)-\Bstar$. A run is saturated at
threshold $\tau$ if $\TriGain\le\tau$. We sweep $\tau\in\{0,0.005,0.01\}$ and
report the reliability $p^{+}(\tau):=\Pr(\TriGain>\tau)$ across paired runs.
For \rvicreg\ against
\vicreg, a rescue at threshold $\tau$ is a paired run with
$\TriGainOf{\vicreg}\le\tau$ and $\TriGainOf{\rvicreg}>\tau$. A harm is
the reverse. Saturation is a run-level failure event, so we treat it as one:
the primary endpoint is how often the third modality helps, not the average
accuracy shift.

\paragraph{Protocol and controls.}
Each method contributes $N=45$ paired runs (nine datasets, five seeds). Each run
pretrains $50$ epochs, selects the checkpoint with lowest pretraining loss,
finetunes $50$ epochs, selects by validation accuracy, and reports test
accuracy. \simsiam\ serves as a second Euclidean baseline: it shares no
variance--covariance regularizer with \vicreg, so saturation under both
objectives implicates Euclidean alignment rather than a regularizer.

\section{Results}
\label{sec:results}

\begin{figure}[tb]
\centering
\includegraphics[width=0.58\linewidth]{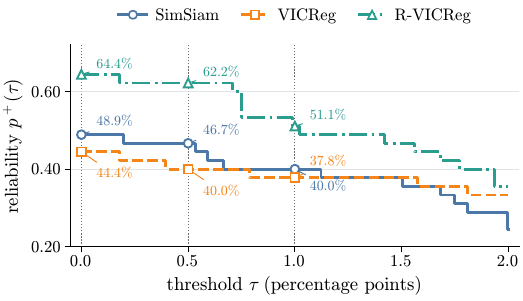}%
\vspace{-2pt}
\caption{Reliability $p^{+}(\tau)=\Pr(\TriGain>\tau)$ over $N=45$ paired runs
per method, swept over $\tau\in\{0,0.005,0.01\}$: \rvicreg\ stays above both
baselines at every threshold, so the ranking is no artifact of $\tau=0$.}
\label{fig:saturation}
\end{figure}

\subsection{Euclidean alignment saturates; \rvicreg\ restores reliability}

Does saturation arise from one objective, or from the geometry? At $\tau=0$,
\vicreg\ saturates in
$55.6\%$ of runs, while \simsiam\ saturates in $51.1\%$. Two Euclidean
objectives with different anti-collapse mechanisms fail at comparable rates.
This result supports the geometric account, but not a regularizer-specific one.
Monotonicity violation is not an edge case; it is the modal outcome.

At $\tau=0$, \rvicreg\ attains $64.4\%$ ($29/45$) against $44.4\%$ ($20/45$) for
\vicreg: a twenty-point shift in the probability that a third modality is worth
adding, at fixed dimension, architecture, and loss weights. The gains cost no
absolute accuracy: mean trimodal accuracy is $0.627$ for \rvicreg, $0.621$ for
\vicreg, and $0.619$ for \simsiam, and Figure~\ref{fig:datasets}(b) reports the
per-dataset paired differences. Figure~\ref{fig:saturation} reports
$p^{+}(\tau)$ across the sweep.

\subsection{The gains are selective, and the harms are real}

Does \rvicreg\ rescue saturated runs without harming unsaturated ones?
Figure~\ref{fig:rescue} (Appendix~\ref{app:geometry}) answers. At $\tau=0$,
\rvicreg\ rescues $13$ of the $25$
runs \vicreg\ saturates and harms $4$ of the $20$ it does not.

Where $\TriGainOf{\vicreg}\le0$, mean paired improvement $\delta$ is $+1.38$
points ($n=25$); elsewhere it is $-0.76$ ($n=20$). This split is the footprint
Section~\ref{sec:theory} predicts; a uniform improvement would have argued
against the capacity account. The method pays off where capacity binds and costs a
little where it does not.

\subsection{What the learned geometry does and does not explain}
\label{sec:secondary}

Does the learned geometry stay hyperbolic in use? It does. The learned curvature
stays negative on every dataset, with mean $\kappa=-0.934$ and range
$[-0.992,-0.820]$, and boundary utilization averages $0.987$. The codes live in
the region where volume grows exponentially, the regime that
Figure~\ref{fig:geometry-showcase} renders.

The per-dataset pattern matches the account. Gains concentrate where
\vicreg\ saturates most, and the two tasks that never saturate show no
difference. Figure~\ref{fig:datasets} gives the breakdown, and
Appendix~\ref{app:geometry} details the runs where \rvicreg\ is worse.

Which saturated runs get rescued? The geometry does not say. Rescued and
not-rescued runs are indistinguishable in curvature magnitude ($0.946$
vs.\ $0.955$) and in boundary utilization ($0.987$ vs.\ $0.983$). The
mechanism evidence is aggregate, not run-level.

\section{Discussion}
\label{sec:discussion}

The intervention modifies one object: the metric in which masked views are
aligned. The rest of the pipeline remains unchanged. The invariance metric therefore co-determines what survives into the
representation. However, if a bimodal-versus-trimodal sweep shows $\TriGain>0$ under the Euclidean
objective, capacity is not binding and switching the metric will more often cost
a little accuracy. \rvicreg\ is conditioned on its diagnostic: a checkable
condition turns geometric theory into a decision rule.

\paragraph{Limitations.}
The real limitation is predictive. Curvature does not reveal rescue: rescued
and not-rescued runs share the same learned geometry
(Section~\ref{sec:secondary}), so the hypothesis holds in aggregate, but we
cannot say in advance which saturated runs a metric change will save. The
rescue is not guaranteed, and the harm is real: $4$ of the $20$ unsaturated
runs got worse. Attribution is incomplete as well: the geodesic objective
couples the two views position-dependently, and separating capacity from that
coupling is left to future work.

\section{Conclusion}
\label{sec:conclusion}

Across nine trimodal datasets under masked-view non-contrastive SSL, the
monotonicity expectation fails in half of the paired runs. We proposed a
geometric account: the invariance metric sets alignment tightness and readout
capacity. Changing only that metric raised reliability from $44.4\%$ to
$64.4\%$, with thirteen rescues against four harms concentrated where Euclidean
alignment saturates. That selectivity is what the proposal predicted.

\bibliographystyle{plainnat}
\bibliography{references}

\appendix

\setlength{\abovedisplayskip}{9pt plus 2pt minus 2pt}
\setlength{\belowdisplayskip}{9pt plus 2pt minus 2pt}
\setlength{\abovedisplayshortskip}{6pt plus 2pt minus 1pt}
\setlength{\belowdisplayshortskip}{7pt plus 2pt minus 1pt}
\setlength{\jot}{5pt}
\setlength{\parskip}{5pt plus 1.5pt minus 1pt}
\setlength{\topsep}{11pt plus 3pt minus 2pt}
\setlength{\partopsep}{3pt plus 1pt minus 1pt}
\renewcommand{\arraystretch}{1.18}

\section{Packing, Covering, and Growth-Law Proofs}
\label{app:prelim}

\begin{definition}[Packing number]
\label{def:packing}
For a metric space $(\mathcal Z,\rho)$ and $S\subseteq\mathcal Z$,
\begin{equation}
\Npack(S,\varepsilon;\rho)
:= \max\Big\{K\in\mathbb N:\ \exists\, z_1,\dots,z_K\in
S,\ \min_{i\neq j}\rho(z_i,z_j)\ge\varepsilon\Big\}.
\end{equation}
\end{definition}

\begin{definition}[Lipschitz functions]
\label{def:lipschitz}
For metric spaces $(\mathcal Z,\rho)$ and $(\mathcal W,\eta)$, $h:\mathcal
Z\to\mathcal W$ is $L$-Lipschitz if $\eta(h(z),h(z'))\le L\rho(z,z')$ for all
$z,z'\in\mathcal Z$.
\end{definition}

\begin{definition}[Metric ball]
For $(\mathcal Z,\rho)$, $z\in\mathcal Z$, $R>0$: $B_{\mathcal
Z}(z,R):=\{u\in\mathcal Z:\rho(u,z)\le R\}$.
\end{definition}

We now prove Theorem~\ref{thm:capacity}.

\begin{proof}
\begin{enumerate}
\item[(i)] Let $h_1,\dots,h_K$ witness readability, and fix $i\neq j$ in
$\{1,\dots,K\}$, $c\in C_i$, and $c'\in C_j$. The function $g:=h_i-h_j$ is
$2L$-Lipschitz as a difference of two $L$-Lipschitz functions, and the
readability condition gives $g(c)\ge1$ and $g(c')\le-1$. Hence
\[
2\;\le\;g(c)-g(c')\;\le\;2L\,d(c,c'),
\]
so $d(c,c')\ge\tfrac{1}{L}$. Choosing one code $c_i\in C_i$ for each
$i\in\{1,\dots,K\}$ therefore produces $K$ points of $B(o,R)$ that are
pairwise at distance at least $\tfrac{1}{L}$, and
$K\le\Npack\big(B(o,R),\tfrac{1}{L};d\big)$ follows from
Definition~\ref{def:packing}.
\item[(ii)] For each $i\in\{1,\dots,K\}$ define
\[
h_i(x)\;:=\;\min\bigl\{1,\;\max\bigl\{0,\;1-L\,(d(x,s_i)-\delta)\bigr\}\bigr\}.
\]
Each $h_i$ is $L$-Lipschitz as a composition of $1$-Lipschitz clamps with an
$L$-Lipschitz function of distance. Fix $i\in\{1,\dots,K\}$ and $c\in C_i$.
The conditions $s_i\in C_i$ and $\diam(C_i)\le\delta$ give $d(c,s_i)\le\delta$,
so $h_i(c)=1$. For $j\in\{1,\dots,K\}$ with $j\neq i$, the triangle inequality
gives
\[
d(c,s_j)\;\ge\;d(s_i,s_j)-d(s_i,c)\;\ge\;\tfrac{1}{L}+2\delta-\delta\;=\;\tfrac{1}{L}+\delta,
\]
so $1-L\,(d(c,s_j)-\delta)\le0$ and $h_j(c)=0$. The readability condition
therefore holds with margin $1$.
\end{enumerate}
\end{proof}

\begin{lemma}[Euclidean ball packing is polynomial in radius]
\label{lem:euclid-packing}
For $\mathbb R^n$ with $d_E(x,y)=\norm{x-y}_2$ and any $R\ge\varepsilon>0$,
\begin{equation}
\Big(\frac{R}{\varepsilon}\Big)^{n}\ \le\ \Npack\big(B_{\mathbb
R^n}(0,R),\varepsilon;d_E\big)\ \le\ \Big(1+\frac{2R}{\varepsilon}\Big)^{n}.
\end{equation}
For the stereographically scaled metric $d_0(x,y)=2\norm{x-y}_2$ the upper bound
becomes $(1+4R/\varepsilon)^n$.
\end{lemma}

\begin{proof}
Upper bound: if $z_1,\dots,z_K$ is an $\varepsilon$-packing of $B_{\mathbb
R^n}(0,R)$, the balls $B(z_i,\varepsilon/2)$ are disjoint and contained in
$B(0,R+\varepsilon/2)$. Comparing volumes gives
\[
K\,(\varepsilon/2)^n\;\le\;(R+\varepsilon/2)^n.
\]
Lower bound: a maximal $\varepsilon$-packing is an $\varepsilon$-cover, since
any point at distance at least $\varepsilon$ from every center could be added
to the packing; and any $\varepsilon$-cover of $B(0,R)$ needs at least
\[
\frac{\mathrm{vol}\big(B(0,R)\big)}{\mathrm{vol}\big(B(0,\varepsilon)\big)}\;=\;\Big(\frac{R}{\varepsilon}\Big)^{n}
\]
balls.
\end{proof}

Both hyperbolic lemmas rest on one imported fact: the volume of a geodesic
ball of radius $t$ in $\mathbb H^n_\kappa$
is~\citep{ratcliffe2006,cheeger2008comparison}
\begin{equation}
\label{eq:hyperbolic-volume}
\mathrm{Vol}\big(B(t)\big)\;=\;\mathrm{Vol}\big(\mathbb S^{n-1}\big)\int_0^{t}\Big(\frac{\sinh(\sqrt{-\kappa}\,s)}{\sqrt{-\kappa}}\Big)^{n-1}ds.
\end{equation}

\begin{lemma}[Hyperbolic ball packing is exponential in radius]
\label{lem:hyperbolic-packing}
Fix $n\ge2$, $\kappa<0$, $\varepsilon>0$. There exist $c_\varepsilon>0$ and
$R_0$ (depending on $n,\kappa,\varepsilon$) such that for all $R\ge R_0$,
\begin{equation}
\Npack\big(B_{\mathbb
H^n_\kappa}(R),\varepsilon\big)\ \ge\ c_\varepsilon\,
\exp\big((n-1)\sqrt{-\kappa}\,R\big).
\end{equation}
\end{lemma}

\begin{proof}
A maximal $\varepsilon$-packing of $B(R)$ is also an $\varepsilon$-cover, since
any point at distance at least $\varepsilon$ from every center could be added
to the packing. All $\varepsilon$-balls have volume
$\mathrm{Vol}(B(\varepsilon))$ by homogeneity, so
\[
\Npack\big(B_{\mathbb H^n_\kappa}(R),\varepsilon\big)\;\ge\;\frac{\mathrm{Vol}\big(B(R)\big)}{\mathrm{Vol}\big(B(\varepsilon)\big)}.
\]
Since $\sinh$ is increasing and $\sinh u\ge e^{u}/4$ for $u\ge\ln2$,
Equation~\eqref{eq:hyperbolic-volume} gives, for
$R\ge R_0:=1+\ln2/\sqrt{-\kappa}$,
\[
\mathrm{Vol}\big(B(R)\big)\;\ge\;\mathrm{Vol}\big(\mathbb S^{n-1}\big)\int_{R-1}^{R}\Big(\frac{\sinh(\sqrt{-\kappa}\,s)}{\sqrt{-\kappa}}\Big)^{n-1}ds\;\ge\;c_{n,\kappa}\,e^{(n-1)\sqrt{-\kappa}\,R},
\]
and the claim follows with
$c_\varepsilon:=c_{n,\kappa}/\mathrm{Vol}(B(\varepsilon))$.
\end{proof}

\begin{lemma}[Hyperbolic ball packing is at most exponential]
\label{lem:hyperbolic-upper}
Fix $n\ge2$, $\kappa<0$, $\varepsilon>0$. There exists $c'_{\varepsilon}>0$
(depending on $n,\kappa,\varepsilon$) such that for all $R\ge1$,
\begin{equation}
\Npack\big(B_{\mathbb
H^n_\kappa}(R),
\varepsilon\big)\ \le\ \frac{\mathrm{Vol}\big(B(R+\varepsilon/2)\big)}{\mathrm{Vol}\big(B(\varepsilon/2)\big)}\ \le\ c'_{\varepsilon}\,\exp\big((n-1)\sqrt{-\kappa}\,R\big).
\end{equation}
\end{lemma}

\begin{proof}
Centers of an $\varepsilon$-packing carry pairwise-disjoint geodesic
$\varepsilon/2$-balls contained in $B(R+\varepsilon/2)$; by homogeneity all
$\varepsilon/2$-balls have equal volume, giving the first inequality. Since
$\sinh u\le e^{u}/2$, Equation~\eqref{eq:hyperbolic-volume} gives
$\mathrm{Vol}(B(t))\le C_{n,\kappa}\,e^{(n-1)\sqrt{-\kappa}\,t}$. Substituting
$t=R+\varepsilon/2$ absorbs $e^{(n-1)\sqrt{-\kappa}\,\varepsilon/2}$ into
$c'_{\varepsilon}$.
\end{proof}

\begin{proposition}[Product-hyperbolic packing lower bound]
\label{prop:product-packing}
Let $\mathcal Z=\prod_{f=1}^{F}\mathbb H^{d}_{\kappa_f}$ with $\kappa_f<0$ and
product metric $\rho(z,z')=(\sum_f\rho_f(z_f,z_f')^2)^{1/2}$. For any sets
$S_f$,
\begin{equation}
\Npack\Big(\prod_f
S_f,\varepsilon;\rho\Big)\ \ge\ \prod_f\Npack(S_f,\varepsilon;\rho_f),
\qquad
B_{\mathcal Z}(o,R)\supseteq\prod_f B_{\mathbb H^d_{\kappa_f}}\big(o_f,R/\sqrt
F\big).
\end{equation}
Consequently there exist $c_\varepsilon>0$, $R_0>0$ such that for $R\ge\sqrt F
R_0$,
\begin{equation}
\Npack\big(B_{\mathcal
Z}(o,R),\varepsilon\big)\ \ge\ (c_\varepsilon)^{F}\exp\Big((d-1)\,\frac{R}{\sqrt
F}\sum_{f=1}^{F}\sqrt{-\kappa_f}\Big).
\end{equation}
\end{proposition}

\begin{proof}
For the packing inequality, take per-factor $\varepsilon$-packings $P_f\subseteq
S_f$; distinct points of $\prod_f P_f$ differ in some factor $f^\star$ by at
least $\varepsilon$, so their product distance is at least $\varepsilon$. For
the ball containment, if $\rho_f(z_f,o_f)\le R/\sqrt F$ for all $f$ then
$\rho(z,o)\le R$. Apply Lemma~\ref{lem:hyperbolic-packing} in each factor at
radius $R/\sqrt F$ and multiply.
\end{proof}

We now prove Corollary~\ref{cor:bits}.

\begin{proof}
Apply Theorem~\ref{thm:capacity} at the two sandwich scales
$\varepsilon_1=\frac{1}{L}$ and $\varepsilon_2=\frac{1}{L}+2\delta$.
\begin{enumerate}
\item[(i)] In $\mathbb R^n$, Lemma~\ref{lem:euclid-packing} bounds both ends
polynomially in $R$. Readability of $K$ concepts forces
\[
K\;\le\;(1+2LR)^{n},
\]
so $R\ge\frac{1}{2L}(K^{1/n}-1)$. A packing of size $(R/\varepsilon_2)^{n}$
exists at scale $\varepsilon_2$, so a radius linear in $\varepsilon_2K^{1/n}$
suffices. In both directions the radius needed for $\log K$ bits scales as
$e^{(\log K)/n}$.
\item[(ii)] In $\mathbb H^n_\kappa$, Lemmas~\ref{lem:hyperbolic-packing}
and~\ref{lem:hyperbolic-upper} bound both ends by
$e^{(n-1)\sqrt{-\kappa}\,R}$ up to constants depending only on
$(n,\kappa,L,\delta)$, so
\[
\log\Npack\;=\;(n-1)\sqrt{-\kappa}\,R+O(1)
\]
at either scale, and the required radius is linear in $\log K$.
\item[(iii)] For the fixed-budget comparison, divide the hyperbolic guarantee
at scale $\varepsilon_2$ by the Euclidean obstruction at scale
$\varepsilon_1$. The ratio is at least
\[
c_{\varepsilon_2}\,e^{(n-1)\sqrt{-\kappa}\,R}\,(1+2LR)^{-n},
\]
which grows exponentially in $R$.
\end{enumerate}
\end{proof}

\section{\texorpdfstring{$\kappa$}{κ}-Stereographic Operations}
\label{app:kappa-ops}

All operations are closed-form and autodiff-compatible.
Table~\ref{tab:kappa-ops} collects them with their flat limits, following the
presentation of the mixed-curvature literature~\citep{skopek2020mixedcurvature}.
Scalar $\kappa$-trigonometry, with the convention $u/\norm{u}_2:=0$ at $u=0$
throughout:
\begin{equation}
\tan_\kappa(r) =
\begin{cases}
\tanh(\sqrt{-\kappa}\,r)/\sqrt{-\kappa} & \kappa<0,\\
r & \kappa=0,
\end{cases}
\qquad
\arctan_\kappa(r) =
\begin{cases}
\mathrm{atanh}(\sqrt{-\kappa}\,r)/\sqrt{-\kappa} & \kappa<0,\\
r & \kappa=0.
\end{cases}
\end{equation}

\begin{table}[h]
\centering
\caption{$\kappa$-stereographic operations used by \rvicreg\ and their flat
limits. Every limit is uniform on compacta (Lemma~\ref{lem:flat-limit}); the
distance limit squares to the invariance-term limit
$d_\kappa(x,y)^{2}\to4\norm{x-y}_2^{2}$ of Proposition~\ref{prop:euclid-limit}.}
\label{tab:kappa-ops}
\begin{tabular}{@{}lll@{}}
\toprule
Operation & Closed form & Limit $\kappa\to0$\\
\midrule
Exponential map at $0$ &
$\exp_0^{\kappa}(u)=\tan_\kappa(\norm{u}_2)\,\dfrac{u}{\norm{u}_2}$ &
$u$\\[2.2ex]
Logarithmic map at $0$ &
$\log_0^{\kappa}(x)=\arctan_\kappa(\norm{x}_2)\,\dfrac{x}{\norm{x}_2}$ &
$x$\\[2.2ex]
M\"obius addition & $x\oplus_\kappa
y=\dfrac{\big(1-2\kappa\ip{x}{y}-\kappa\norm{y}_2^2\big)x+\big(1+\kappa\norm{x}_2^2\big)y}{1-2\kappa\ip{x}{y}+\kappa^2\norm{x}_2^2\norm{y}_2^2}$
& $x+y$\\[2.6ex]
Geodesic distance & $d_\kappa(x,y)=2\arctan_\kappa\!\big(\norm{(-x)\oplus_\kappa
y}_2\big)$ & $2\norm{x-y}_2$\\[1.4ex]
Conformal factor & $\lambda_\mu^\kappa=2/(1+\kappa\norm{\mu}_2^2)$ &
$2$\\[1.4ex]
Logarithmic map at $\mu$ &
$\log_\mu^{\kappa}(x)=\dfrac{2}{\lambda_\mu^{\kappa}}\,
\log_0^{\kappa}\big((-\mu)\oplus_\kappa
x\big)$ & $x-\mu$\\[2.2ex]
Exponential map at $\mu$ &
$\exp_\mu^{\kappa}(v)=\mu\oplus_\kappa\exp_0^{\kappa}\big(\lambda_\mu^{\kappa}v/2\big)$
& $\mu+v$\\[1.4ex]
Karcher step &
$\mu\leftarrow\exp_\mu^{\kappa}\big(\tfrac{1}{2B}\sum_{j}\log_\mu^{\kappa}(q_j)\big)$
& $\mu\leftarrow\tfrac{1}{2B}\sum_{j}q_j$\\
\bottomrule
\end{tabular}
\end{table}

The exponential map at $\mu$ is the one used by the Karcher iteration of
Section~\ref{sec:method}. The next lemma proves the limits and covers the rest
of the objective.

\begin{lemma}[Full-objective flat limit]
\label{lem:flat-limit}
For one factor, let $q_1,\dots,q_{2B}$ be the factor codes of both views and
$\bar q:=\frac{1}{2B}\sum_{j}q_j$. Set $\mu^{(0)}=0$ and, for $s\ge0$,
\[
\mu^{(s+1)}\;=\;\exp_{\mu^{(s)}}^{\kappa}\Big(\frac{1}{2B}\sum_{j=1}^{2B}\log_{\mu^{(s)}}^{\kappa}(q_j)\Big).
\]
For any fixed number of iterations $t\ge1$, as $\kappa\to0$,
\begin{equation}
\mu^{(t)}\longrightarrow\bar q,
\qquad
t_{i,b}^{(f)}=\log_{\mu^{(t)}}^{\kappa}\big(x_{i,b}^{(f)}\big)\longrightarrow
z_{i,b}^{(f)}-\bar z_f,
\end{equation}
where $\bar z_f$ is the Euclidean mean of both views' factor codes. Consequently
$\mathcal L_{\mathrm{var}}$ and $\mathcal L_{\mathrm{cov}}$ converge to their
\vicreg\ values and, with the invariance limit above, $\mathcal
L_{\mathrm{rvicreg}}\to\mathcal L_{\mathrm{vicreg}}$.
\end{lemma}

\begin{proof}
All codes in the fixed minibatch have bounded norms, and only finitely many
iterations are taken, so every quantity below ranges over a compact set and
every limit is uniform on compacta. As $\kappa\to0$, the scalar maps
$\tan_\kappa,\arctan_\kappa\to\mathrm{id}$ uniformly on compacta, hence
$\exp_0^{\kappa},\log_0^{\kappa}\to\mathrm{id}$; the conformal factor satisfies
$\lambda_\mu^{\kappa}\to2$; and the M\"obius formula gives $x\oplus_\kappa y\to
x+y$. Therefore
$\log_\mu^{\kappa}(x)=\frac{2}{\lambda_\mu^{\kappa}}\log_0^{\kappa}\big((-\mu)\oplus_\kappa
x\big)\to x-\mu$ and
$\exp_\mu^{\kappa}(v)=\mu\oplus_\kappa\exp_0^{\kappa}(\lambda_\mu^{\kappa}v/2)\to\mu+v$.
For the iterates: $\mu^{(1)}\to0+\frac{1}{2B}\sum_j(q_j-0)=\bar q$, and if
$\mu^{(s)}\to\bar q$ then $\log_{\mu^{(s)}}^{\kappa}(q_j)\to q_j-\bar q$, whose
average tends to $0$, so $\mu^{(s+1)}\to\bar q+0=\bar q$; induction gives
$\mu^{(t)}\to\bar q$ for every fixed $t\ge1$.
For the tangent codes: $x_{i,b}^{(f)}=\exp_0^{\kappa}(z_{i,b}^{(f)})\to
z_{i,b}^{(f)}$ and $\bar q\to\bar z_f$, so $t_{i,b}^{(f)}\to z_{i,b}^{(f)}-\bar
z_f$; concatenation over factors gives $t_{i,b}\to z_{i,b}-\bar z$.
Finally, subtracting the common batch vector $\bar z$ from every column changes
neither per-dimension standard deviations nor covariance matrices, and both are
continuous in the finitely many tangent codes, so $\mathcal L_{\mathrm{var}}$
and $\mathcal L_{\mathrm{cov}}$ converge to the \vicreg\ values computed on
$z_1,z_2$. The detachment of $\mu$ is immaterial in the limit: it is a
batch-dependent constant whose only effect is the common shift. The weighted sum
then converges term by term.
\end{proof}

\section{The Canonical Geometry of Location--Uncertainty Codes}
\label{app:fisher}

This appendix motivates, from information geometry, why a product of learnable
negative-curvature factors is a canonical geometry for codes that carry
per-factor location and uncertainty. It is motivation for the design;
Theorem~\ref{thm:capacity} and all empirical claims are independent of it.

The modeling axiom is stated explicitly: suppose each factor's code is read as a
statistical experiment. It is a location--scale state $(\mu,\sigma)$ with
$\mu\in\mathbb R^{m}$ and $\sigma>0$. It combines a fitted position with its
alignment uncertainty. Suppose the geometry placed on codes must be invariant
under sufficient statistics and monotone under Markov morphisms. Data processing
cannot increase statistical distinguishability.

\begin{proposition}[Fisher--Rao geometry of location--uncertainty codes]
\label{prop:fisher}
Under the modeling axiom, the following classical facts apply.
\begin{enumerate}
\item[(i)] Invariant metric: the metric is the Fisher--Rao metric, uniquely up
to a positive scale on each independent
factor~\citep{chentsov1982statistical,ay2017information}.
\item[(ii)] Gaussian factors are hyperbolic: for the isotropic Gaussian family
$\mathcal N(\mu,\sigma^2I_m)$ on $\mathbb R^{m}$, the Fisher--Rao line element
is
\[
ds^2\;=\;\frac{\norm{d\mu}_2^2+2m\,d\sigma^2}{\sigma^2},
\]
and the substitution $y=\sqrt{2m}\,\sigma$ gives
\[
ds^2\;=\;2m\,\frac{\norm{d\mu}_2^2+dy^2}{y^2},
\]
the Poincar\'e half-space metric scaled by $2m$, hence $\mathbb H^{m+1}$ with
constant curvature
$-1/(2m)$~\citep{atkinson1981rao,skovgaard1984riemannian,costa2015fisher}.
General regular location--scale families give scaled hyperbolic
factors~\citep{atkinson1981rao}.
\item[(iii)] Independent factors compose: a product of $F$ independent
location--scale factors carries the product Fisher--Rao geometry
$\prod_{f}\mathbb H^{m_f+1}_{\kappa_f}$ with $\kappa_f<0$, and the per-factor
scale freedom of (i) is one curvature magnitude per factor. A learnable
$\kappa_f$ parameterizes that freedom, and only that freedom.
\end{enumerate}
\end{proposition}

\begin{corollary}[The fixed-uncertainty slice is flat]
\label{cor:flat-slice}
On the slice of a factor with $\sigma$ held fixed, the induced metric is
\[
ds^2\big|_{d\sigma=0}\;=\;\frac{\norm{d\mu}_2^2}{\sigma^2},
\]
which is flat. Euclidean \vicreg\ is, in this reading, the Fisher--Rao geometry
of codes whose uncertainty coordinate is frozen; unfreezing that coordinate
forces negative curvature.
\end{corollary}

Within this reading, tightening alignment (smaller $\sigma$) moves codes toward
the ideal boundary and expands the accessible Fisher volume; the
boundary-utilization statistic of Section~\ref{sec:method} measures the
occupancy of that region.

\section{Per-dataset Reliability and Learned Geometry}
\label{app:geometry}

The aggregate reliability hides substantial dataset heterogeneity.
Figure~\ref{fig:datasets} (Section~\ref{sec:results}) reports the per-dataset
positive-gain fractions and paired trimodal-accuracy differences,
Figure~\ref{fig:rescue} shows
the paired quadrants behind the rescue counts, and
Figure~\ref{fig:app-diagnostics} reports the learned geometry.

\paragraph{Where the gains are, and where they are not.}
The per-dataset positive-gain advantage of \rvicreg\ over \vicreg\ is $+0.80$ on
CSGO, $+0.60$ on Pokemon Secondary, $+0.40$ on PetFinder, and $+0.20$ on Pokemon
Primary. The two largest Hearthstone tasks, CardClass and Set, are positive in
every seed for every method and show no difference: no capacity problem, no
effect. On Hearthstone Spell, \rvicreg\ is worse ($-0.20$), as the capacity
account predicts where flat supply is not binding.

\begin{figure}[tb]
\centering
\includegraphics[width=\linewidth]{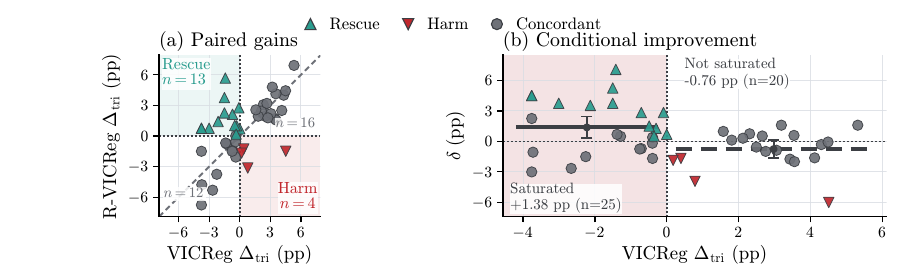}
\caption{A rescue is a paired run that \vicreg\ saturates ($\TriGain\le\tau$)
and \rvicreg\ does not; a harm is the reverse. (a) The quadrants at $\tau=0$
hold $13$ rescues, $4$ harms, and $28$ concordant runs. (b) The paired
improvement $\delta$ is positive where \vicreg\ saturates and negative where it
does not; bars show $95\%$ Student-$t$ intervals.}
\label{fig:rescue}
\end{figure}

\paragraph{The distribution improves in the bulk, not the extreme tail.}
\rvicreg\ improves the mean ($+0.008$ vs.\ $+0.004$), median ($+0.010$
vs.\ $-0.003$), and $10$th percentile ($-0.027$ vs.\ $-0.034$) of $\TriGain$,
but its $5$th percentile ($-0.046$) and minimum ($-0.068$) are worse than
\vicreg's ($-0.038$): the intervention shifts the typical case, not the worst
case. Percentiles use linear interpolation.

\begin{figure}[tb]
\centering
\includegraphics[width=\linewidth]{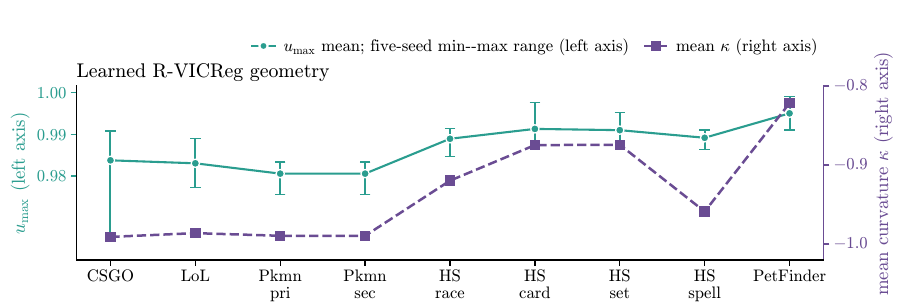}
\caption{The learned \rvicreg\ geometry stays hyperbolic on every dataset:
curvature remains negative and boundary utilization stays high. The left axis
reports $u_{\max}$, with bars spanning the five-seed minimum--maximum range;
the right axis reports the mean learned curvature $\kappa$.}
\label{fig:app-diagnostics}
\end{figure}

\paragraph{Seeing the learned metric.}
Figure~\ref{fig:geometry-showcase} makes one learned geometry visible: the
trimodal CSGO run at seed 456. Panel~(a) draws
the flat \vicreg\ projector with the angle from the top two principal directions
and the radius equal to the true embedding norm; rings at equal distances stay
evenly spaced. Panel~(b) draws the most-utilized \rvicreg\ factor
($\kappa_{32}=-0.988$) with the angle from the top two tangent principal
directions ($67\%$ of the variance) and the radius equal to the exact geodesic
distance from the origin. The same rings now crowd toward the boundary circle.
The crowding is the exponential volume of Section~\ref{sec:theory} at work:
$71\%$ of the points lie beyond geodesic radius $2$, and at the median occupied
radius the factor's geodesic sphere carries $11.4$ times the flat-metric area.
Panel~(c) lifts the same points onto the constant-curvature hyperboloid above
their Poincar\'e-ball shadow.

\section{Algorithm, Hyperparameters, and Reproducibility}
\label{app:hyperparams}

\begin{algorithm2e}[h]
\caption{\rvicreg\ masked-view loss $\mathcal L_{\mathrm{MMM}}$ (one minibatch)}
\label{alg:rvicreg}
\DontPrintSemicolon
construct views $a_1(x)$, $a_2(x)$; encode, fuse, and project to
$z_1,z_2\in\mathbb R^{D}$\;
reshape $z_i$ into factors $z_i^{(f)}\in\mathbb R^{d}$, $f=1,\dots,F$\;
map $x_i^{(f)}\leftarrow\exp_0^{\kappa_f}(z_i^{(f)})$ and compute $\mathcal
L_{\mathrm{inv}}$ by Equation~\eqref{eq:rvicreg-inv}\;
\For{$f=1$ \KwTo $F$}{
  $\mu_f\leftarrow\FM\big(\mathrm{cat}(x_1^{(f)},x_2^{(f)})\big)$ (three
origin-initialized Karcher iterations, detached)\;
  $t_i^{(f)}\leftarrow\log_{\mu_f}^{\kappa_f}(x_i^{(f)})$ for $i\in\{1,2\}$\;
}
concatenate $t_i\leftarrow[t_i^{(1)};\dots;t_i^{(F)}]$; compute $\mathcal
L_{\mathrm{var}},\mathcal L_{\mathrm{cov}}$ by
Equation~\eqref{eq:rvicreg-varcov}\;
\Return $\lambda_{\mathrm{inv}}\mathcal
L_{\mathrm{inv}}+\lambda_{\mathrm{var}}\mathcal
L_{\mathrm{var}}+\lambda_{\mathrm{cov}}\mathcal L_{\mathrm{cov}}$\;
\end{algorithm2e}

\begin{table}[h]
\centering
\caption{Hyperparameters shared by all experiments.}
\label{tab:hyperparams}
\begin{tabular}{ll}
\toprule
Hyperparameter & Value \\
\midrule
Pretrain / finetune epochs & $50$ / $50$ \\
Batch size & $64$ \\
Optimizer / learning rate / weight decay & AdamW / $10^{-4}$ / $0.02$ \\
Warmup epochs / gradient clip & $5$ / $1.0$ \\
Text / image / tabular reconstruction weights & $(1.0,\,1.0,\,0.01)$ \\
Alignment weight $\lambda_{\mathrm{MMM}}$ & $1.0$ \\
$(\lambda_{\mathrm{inv}},\lambda_{\mathrm{var}},\lambda_{\mathrm{cov}})$,
$\gamma$ & $(25,\,25,\,1)$, $1.0$ \\
Projection dimension $D$ & $256$ \\
\rvicreg\ factors $F\times d$ & $64\times4$ \\
Curvature bound and initialization & $\kappa_{\max}=1.5$, $\kappa_0=-1$ \\
Fr\'echet mean steps (detached) & $3$ \\
Seeds per dataset and method & $5$ \\
\bottomrule
\end{tabular}
\end{table}

All datasets are public (MuG-Bench and PetFinder); MuG-Bench uses the provided
splits and PetFinder an $80/5/15$ split. Image and text encoders are initialized
from pretrained checkpoints and finetuned during multimodal pretraining;
\vicreg\ and \rvicreg\ share a VICReg-style projector with BatchNorm, while
\simsiam\ uses its own projector and predictor. The downstream classifier
operates on the fusion CLS embedding without the SSL projector. The full sweep
(nine datasets, seven modality subsets, three methods, five seeds) comprises
$945$ completed runs whose recorded run times sum to approximately $200$
GPU-hours, executed on two NVIDIA RTX 4090 GPUs. The code will be
open-sourced.

\end{document}